%% file: arxiv.tex
\documentclass{article}
\usepackage[margin=1in]{geometry}
\usepackage[utf8]{inputenc} 
\usepackage[T1]{fontenc}    
\usepackage[colorlinks=true,allcolors=magenta]{hyperref}       
\usepackage{url}            
\usepackage{booktabs}       
\usepackage{amsfonts}       
\usepackage{nicefrac}       
\usepackage{microtype}      
\usepackage[numbers,compress]{natbib}

\usepackage[affil-sl]{authblk}
\title{Linear regression without correspondence}
\author[1]{Daniel Hsu}
\author[1]{Kevin Shi}
\author[2]{Xiaorui Sun}
\affil[1]{Columbia University, New York, NY}
\affil[2]{Microsoft Research, Redmond, WA}

\usepackage{amsmath,amsbsy,amsfonts,amssymb,amsthm,commath,dsfont,stmaryrd}
\usepackage[usenames]{color}
\usepackage{algorithm,algorithmic}
\usepackage{cleveref}
\usepackage{xspace}
\usepackage{enumitem}

\def\ddefloop#1{\ifx\ddefloop#1\else\ddef{#1}\expandafter\ddefloop\fi}
\def\ddef#1{\expandafter\def\csname bb#1\endcsname{\ensuremath{\mathbb{#1}}}}
\ddefloop ABCDEFGHIJKLMNOPQRSTUVWXYZ\ddefloop
\def\ddef#1{\expandafter\def\csname c#1\endcsname{\ensuremath{\mathcal{#1}}}}
\ddefloop ABCDEFGHIJKLMNOPQRSTUVWXYZ\ddefloop
\def\ddef#1{\expandafter\def\csname v#1\endcsname{\ensuremath{\boldsymbol{#1}}}}
\ddefloop ABCDEFGHIJKLMNOPQRSTUVWXYZabcdefghijklmnopqrstuvwxyz\ddefloop
\def\veps{\ensuremath{\varepsilon}} 
\def\ddef#1{\expandafter\def\csname v#1\endcsname{\ensuremath{\boldsymbol{\csname #1\endcsname}}}}
\ddefloop {alpha}{beta}{gamma}{delta}{epsilon}{varepsilon}{veps}{zeta}{eta}{theta}{vartheta}{iota}{kappa}{lambda}{mu}{nu}{xi}{pi}{varpi}{rho}{varrho}{sigma}{varsigma}{tau}{upsilon}{phi}{varphi}{chi}{psi}{omega}{Gamma}{Delta}{Theta}{Lambda}{Xi}{Pi}{Sigma}{varSigma}{Sig}{Upsilon}{Phi}{Psi}{Omega}{ell}\ddefloop

\renewcommand\v[1]{{\ensuremath{\boldsymbol{#1}}}} 
\newcommand\T{{\ensuremath{\scriptscriptstyle{\top}}}} 
\newcommand{\E}{\ensuremath{\mathbb{E}}} 
\newcommand\cov{\ensuremath{\operatorname{cov}}} 
\newcommand\poly{\ensuremath{\operatorname{poly}}} 
\newcommand\rank{\ensuremath{\operatorname{rank}}} 
\newcommand\ind[1]{\ensuremath{\mathds{1}\{#1\}}} 

\newcommand{\R}{\ensuremath{\mathbb{R}}} 
\newcommand{\Z}{\ensuremath{\mathbb{Z}}} 
\newcommand\Normal{\ensuremath{\operatorname{N}}} 
\newcommand\snr{\ensuremath{\mathsf{SNR}}} 
\newcommand\mle{\ensuremath{\mathsf{mle}}} 
\newcommand\tv{\ensuremath{\mathsf{tv}}} 
\newcommand\bag[1]{\ensuremath{\Lbag #1 \Rbag}} 
\newcommand\KL{\ensuremath{\operatorname{KL}}} 
\newcommand\opt{\ensuremath{\operatorname{opt}}} 
\newcommand\argmin{\mathop{\arg\min}} 
\newcommand\PLS{\textsc{Permuted Linear System}\xspace} 
\newcommand\TP{\textsc{$3$-Partition}\xspace} 
\newcommand\PARTITION{\textsc{Partition}\xspace} 

\newtheorem{lemma}{Lemma}
\newtheorem{proposition}{Proposition}
\newtheorem{theorem}{Theorem}
\crefname{proposition}{proposition}{propositions}
\theoremstyle{definition}
\newtheorem{remark}{Remark}

\begin{document}

\maketitle

\begin{abstract}%
  \input{abstract}
\end{abstract}

\input{intro}

\input{approx}

\input{lattice}

\input{lower}

\section*{Acknowledgments}

We are grateful to Ashwin Pananjady, Micha{\l} Derezi{\'n}ski, and Manfred Warmuth for helpful discussions.
DH was supported in part by NSF awards DMR-1534910 and IIS-1563785, a Bloomberg Data Science Research Grant, and a Sloan Research Fellowship.
XS was supported in part by a grant from the Simons Foundation (\#320173 to Xiaorui Sun).
This work was done in part while DH and KS were research visitors and XS was a
research fellow at the Simons Institute for the Theory of Computing.

\bibliographystyle{plainnat}
\bibliography{main}

\newpage

\appendix

\input{appendix-npc}

\input{appendix-approx}
\input{appendix-prob}

\input{appendix-lower}

\end{document}

%% file: abstract.tex
This article considers algorithmic and statistical aspects of linear regression
when the correspondence between the covariates and the responses is unknown.
First, a fully polynomial-time approximation scheme is given for the natural
least squares optimization problem in any constant dimension.
Next, in an average-case and noise-free setting where the responses exactly
correspond to a linear function of i.i.d.~draws from a standard multivariate
normal distribution, an efficient algorithm based on lattice basis reduction is
shown to exactly recover the unknown linear function in arbitrary dimension.
Finally, lower bounds on the signal-to-noise ratio are established for
approximate recovery of the unknown linear function by any estimator.

%% file: intro.tex
\section{Introduction}
\label{sec:intro}

Consider the problem of recovering an unknown vector $\bar{\vw} \in \R^d$ from
noisy linear measurements when the correspondence between the measurement
vectors and the measurements themselves is unknown.
The measurement vectors (i.e., covariates) from $\R^d$ are denoted by $\vx_1,
\vx_2, \dotsc, \vx_n$; for each $i \in [n] := \cbr[0]{1,2,\dotsc,n}$, the $i$-th
measurement (i.e., response) $y_i$ is obtained using $\vx_{\bar{\pi}(i)}$:
\begin{equation}
  y_i
  \ = \
  \bar{\vw}^\T \vx_{\bar{\pi}(i)}
  + \veps_i
  \,,
  \quad i \in [n]
  \,.
  \label{eq:measurements}
\end{equation}
Above, $\bar{\pi}$ is an unknown permutation on $[n]$, and the $\veps_1,
\veps_2, \dotsc, \veps_n$ are unknown measurement errors.

This problem (which has been called \emph{unlabeled
sensing}~\citep{unnikrishnan2015unlabeled}, \emph{linear regression with an
unknown permutation}~\citep{pananjady2016linear}, and \emph{linear regression
with shuffled labels}~\citep{abid2017linear}) arises in many settings.
For example, physical sensing limitations may create ambiguity in or lose the
ordering of measurements.
Or, the covariates and responses may be derived from separate databases that
lack appropriate record linkage (perhaps for privacy reasons).
See the aforementioned references for more details on these applications.
The problem is also interesting because the missing correspondence makes an
otherwise well-understood problem into one with very different computational
and statistical properties.

\paragraph{Prior works.}

\citet{unnikrishnan2015unlabeled} study conditions on the measurement vectors
that permit recovery of any target vector $\bar{\vw}$ under noiseless
measurements.
They show that when the entries of the $\vx_i$ are drawn i.i.d.~from a
continuous distribution, and $n \geq 2d$, then almost surely, every vector
$\bar{\vw} \in \R^d$ is uniquely determined by noiseless correspondence-free
measurements as in~\eqref{eq:measurements}.
(Under noisy measurements, it is shown that $\bar{\vw}$ can be recovered when
an appropriate signal-to-noise ratio tends to infinity.)
It is also shown that $n \geq 2d$ is necessary for such a guarantee that holds
\emph{for all} vectors $\bar{\vw} \in \R^d$.

\citet{pananjady2016linear} study statistical and computational limits on
recovering the unknown permutation $\bar{\pi}$.
On the statistical front, they consider necessary and sufficient conditions on
the signal-to-noise ratio $\snr := \norm{\bar{\vw}}_2^2 / \sigma^2$ when the
measurement errors $(\veps_i)_{i=1}^n$ are i.i.d.~draws from the normal
distribution $\Normal(0,\sigma^2)$ and the measurement vectors $(\vx_i)_{i=1}^n$
are i.i.d.~draws from the standard multivariate normal distribution
$\Normal(\v0,\vI_d)$.
Roughly speaking, exact recovery of $\bar{\pi}$ is possible via maximum
likelihood when $\snr \geq n^c$ for some absolute constant $c>0$, and
approximate recovery is impossible for any method when $\snr \leq n^{c'}$ for
some other absolute constant $c'>0$.
On the computational front, they show that the least squares problem (which is
equivalent to maximum likelihood problem)
\begin{equation}
  \min_{\vw,\pi} \sum_{i=1}^n \del{ \vw^\T\vx_{\pi(i)} - y_i }^2
  \label{eq:mle}
\end{equation}
given arbitrary $\vx_1, \vx_2, \dotsc, \vx_n \in \R^d$ and $y_1, y_2, \dotsc,
y_n \in \R$ is NP-hard when $d = \Omega(n)$\footnote{%
  \citet{pananjady2016linear} prove that \PARTITION reduces to the problem of deciding if the optimal value of \eqref{eq:mle} is zero or non-zero. Note that \PARTITION is weakly, but not strongly, NP-hard: it admits a pseudo-polynomial-time algorithm~\citep[Section 4.2]{garey1979computers}. In \Cref{sec:npc}, we prove that the least squares problem is strongly NP-hard by reduction from \TP (which is strongly NP-complete~\citep[Section 4.2.2]{garey1979computers}).%
}, but admits a polynomial-time
algorithm (in fact, an $O(n \log n)$-time algorithm based on sorting) when $d =
1$.

\citet{abid2017linear} observe that the maximum likelihood estimator can be inconsistent for estimating $\bar{\vw}$ in certain settings (including the normal setting of \citet{pananjady2016linear}, with $\snr$ fixed but $n\to\infty$).
One of the alternative estimators they suggest is consistent under additional assumptions in dimension $d=1$.
\citet{elhami2017unlabeled} give a $O(dn^{d+1})$-time algorithm that, in
dimension $d = 2$, is guaranteed to approximately recover $\bar{\vw}$ when the
measurement vectors are chosen in a very particular way from the unit circle and
the measurement errors are uniformly bounded.

\paragraph{Contributions.}

We make progress on both computational and statistical aspects of the problem.

\begin{enumerate}[leftmargin=2em]
  \item
    We give an approximation algorithm for the least squares problem from
    \eqref{eq:mle} that, any given $(\vx_i)_{i=1}^n$, $(y_i)_{i=1}^n$, and
    $\epsilon \in \intoo{0,1}$, returns a solution with objective value at most
    $1+\epsilon$ times that of the minimum in time $(n/\epsilon)^{O(d)}$.
    This a fully polynomial-time approximation scheme for any constant
    dimension.

  \item
    We give an algorithm that exactly recovers $\bar{\vw}$ in the measurement
    model from~\eqref{eq:measurements}, under the assumption that there are no
    measurement errors and the covariates $(\vx_i)_{i=1}^n$ are i.i.d.~draws
    from $\Normal(\v0,\vI_d)$.
    The algorithm, which is based on a reduction to a lattice problem and
    employs the lattice basis reduction algorithm of
    \citet{lenstra1982factoring}, runs in $\poly(n,d)$ time when the covariate
    vectors $(\vx_i)_{i=1}^n$ and target vector $\bar{\vw}$ are appropriately
    quantized.
    This result may also be regarded as \emph{for each}-type guarantee for
    exactly recovering a fixed vector $\bar{\vw}$, which complements the
    \emph{for all}-type results of \citet{unnikrishnan2015unlabeled} concerning
    the number of measurement vectors needed for recovering all possible
    vectors.

  \item
    We show that in the measurement model from~\eqref{eq:measurements} where the
    measurement errors are i.i.d.~draws from $\Normal(0,\sigma^2)$ and the
    covariate vectors are i.i.d.~draws from $\Normal(\v0,\vI_d)$, then no
    algorithm can approximately recover $\bar{\vw}$ unless $\snr \geq
    C\min\cbr[0]{1 ,\, d / \log\log(n)}$ for some absolute constant $C>0$.
    We also show that when the covariate vectors are i.i.d.~draws from the
    uniform distribution on $[-1/2,1/2]^d$, then approximate recovery is
    impossible unless $\snr \geq C'$ for some other absolute constant $C'>0$.

\end{enumerate}

Our algorithms are not meant for practical deployment, but instead are intended
to shed light on the computational difficulty of the least squares problem and
the average-case recovery problem.
Indeed, note that a na\"ive brute-force search over permutations requires time
$\Omega(n!) = n^{\Omega(n)}$, and the only other previous algorithms (already
discussed above) were restricted to $d=1$~\citep{pananjady2016linear} or only
had some form of approximation guarantee when $d=2$~\citep{elhami2017unlabeled}.
We are not aware of previous algorithms for the average-case problem in general
dimension $d$.\footnote{A
recent algorithm of \citet{pananjady2017denoising} exploits a similar
average-case setting but only for a somewhat easier variant of the problem where
more information about the unknown correspondence is provided.}

Our lower bounds on $\snr$ stand in contrast to what is achievable in the
classical linear regression model (where the covariate/response
correspondence is known): in that model, the $\snr$ requirement for
approximately recovering $\bar{\vw}$ scales as $d/n$, and hence the problem
becomes easier with $n$.
The lack of correspondence thus drastically changes the difficulty of the
problem.

%% file: approx.tex
\section{Approximation algorithm for the least squares problem}
\label{sec:approx}

In this section, we consider the least squares problem from \Cref{eq:mle}.
The inputs are an arbitrary matrix $\vX = [ \vx_1 | \vx_2 | \dotsb | \vx_n ]^\T
\in \R^{n \times d}$ and an arbitrary vector $\vy = (y_1,y_2,\dotsc,y_n)^\T \in
\R^n$, and the goal is to find a vector $\vw \in \R^d$ and permutation matrix
$\vPi \in \cP_n$
(where $\cP_n$ denotes the space of $n \times n$ permutation matrices\footnote{%
  Each permutation matrix $\vPi \in \cP_n$ corresponds to a permutation $\pi$ on
  $[n]$; the $(i,j)$-th entry of $\vPi$ is one if $\pi(i) = j$ and is zero
  otherwise.%
}) to minimize $\norm[0]{\vX\vw - \vPi^\T\vy}_2^2$.
This problem is NP-hard in the case where $d = \Omega(n)$
\citep{pananjady2016linear} (see also \Cref{sec:npc}).
We give an approximation scheme that, for any $\epsilon \in \intoo{0,1}$,
returns a $(1+\epsilon)$-approximation in time $(n/\epsilon)^{O(k)} +
\poly(n,d)$, where $k := \rank(\vX) \leq \min\{n,d\}$.

We assume without loss of generality that $\vX \in \R^{n \times k}$ and $\vX^\T
\vX = \vI_k$.
This is because we can always replace $\vX$ with its matrix of left singular
vectors $\vU \in \R^{n \times k}$, obtained via singular value decomposition
$\vX = \vU \vvarSigma \vV^\T$, where $\vU^\T\vU = \vV^\T\vV = \vI_k$ and
$\vvarSigma \succ 0$ is diagonal.
A solution $(\vw,\vPi)$ for $(\vU,\vy)$ has the same cost as the solution
$(\vV\vvarSigma^{-1}\vw,\vPi)$ for $(\vX,\vy)$,
and a solution $(\vw,\vPi)$ for $(\vX,\vy)$ has the same cost as the solution
$(\vvarSigma\vV^\T\vw,\vPi)$ for $(\vU,\vy)$.

\subsection{Algorithm}

\begin{algorithm}[t]
  \renewcommand\algorithmicrequire{\textbf{input}}
  \renewcommand\algorithmicensure{\textbf{output}}
  \caption{Approximation algorithm for least squares problem}
  \label{alg:approx}
  \begin{algorithmic}[1]
    \REQUIRE
    Covariate matrix $\vX = [ \vx_1 | \vx_2 | \dotsb | \vx_n ]^\T \in \R^{n
    \times k}$;
    response vector $\vy = (y_1,y_2,\dotsc,y_n)^\T \in \R^n$;
    approximation parameter $\epsilon \in \intoo{0,1}$.

    \renewcommand\algorithmicrequire{\textbf{assume}}
    \REQUIRE
    $\vX^\T\vX = \vI_k$.

    \ENSURE
    Weight vector $\hat{\vw} \in \R^k$ and permutation matrix $\hat{\vPi} \in
    \cP_n$.

    \STATE
    Run ``Row Sampling'' algorithm with input matrix $\vX$ to obtain a matrix
    $\vS \in \R^{r \times n}$ with $r = 4k$.

    \STATE
    Let $\cB$ be the set of vectors $\vb = (b_1,b_2,\dotsc,b_n)^\T
    \in \R^n$ satisfying the following: for each $i \in [n]$,
    \begin{itemize}[leftmargin=2em]
      \item
        if the $i$-th column of $\vS$ is all zeros, then $b_i = 0$;

      \item
        otherwise, $b_i \in \cbr[0]{ y_1, y_2, \dotsc, y_n}$.

    \end{itemize}

    \STATE Let $c := 1+4(1 + \sqrt{n/(4k)})^2$.

    \FOR{each $\vb \in \cB$}

      \STATE
      Compute $\tilde\vw_{\vb} \in \argmin_{\vw \in \R^k} \norm[0]{\vS(\vX\vw -
      \vb)}_2^2$, and let $r_{\vb} := \min_{\vPi \in \cP_n}
      \norm[0]{\vX\tilde\vw_{\vb} - \vPi^\T\vy}_2^2$.

      \STATE
      Construct a $\sqrt{\epsilon r_{\vb}/c}$-net $\cN_{\vb}$ for the Euclidean
      ball of radius $\sqrt{cr_{\vb}}$ around $\tilde\vw_{\vb}$, so that for
      each $\vv \in \R^k$ with $\norm[0]{\vv - \tilde\vw_{\vb}}_2 \leq
      \sqrt{cr_{\vb}}$, there exists $\vv' \in \cN_{\vb}$ such that
      $\norm[0]{\vv - \vv'}_2 \leq \sqrt{\epsilon r_{\vb}/c}$.

    \ENDFOR

    \RETURN
    $\displaystyle\hat\vw \in \argmin_{\vw \in \bigcup_{\vb \in \cB} \cN_{\vb}}
    \min_{\vPi \in \cP_n} \norm[0]{\vX\vw - \vPi^\T\vy}_2^2$
    and
    $\displaystyle\hat\vPi \in \argmin_{\vPi \in \cP_n} \norm[0]{\vX\hat\vw -
    \vPi^\T\vy}_2^2$.

  \end{algorithmic}
\end{algorithm}

Our approximation algorithm, shown as \Cref{alg:approx}, uses a careful
enumeration to beat the na\"ive brute-force running time of $\Omega(|\cP_n|) =
\Omega(n!)$.
It uses as a subroutine a ``Row Sampling'' algorithm of
\citet{boutsidis2013near} (described in \Cref{sec:approx-details}), which has
the following property.

\begin{theorem}[Specialization of Theorem 12 in \citep{boutsidis2013near}]
  \label{thm:row-sampling}
  There is an algorithm (``Row Sampling'') that, given any matrix $\vA \in \R^{n
  \times k}$ with $n\geq k$, returns in $\poly(n,k)$ time a matrix $\vS \in
  \R^{r \times n}$ with $r = 4k$ such that the following hold.
  \begin{enumerate}[leftmargin=2em,itemsep=0ex]
    \item
      Every row of $\vS$ has at most one non-zero entry.

    \item
      For every $\vb \in \R^n$, every $\vw' \in \argmin_{\vw \in \R^k}
      \norm[0]{\vS(\vA\vw - \vb)}_2^2$ satisfies $\norm[0]{\vA\vw' - \vb}_2^2
      \leq c \cdot \min_{\vw \in \R^k} \norm[0]{\vA\vw - \vb}_2^2$ for $c = 1 +
      4(1+\sqrt{n/(4k)})^2 = O(n/k)$.

  \end{enumerate}
\end{theorem}

The matrix $\vS$ returned by Row Sampling determines a (weighted) subset of
$O(k)$ rows of $\vA$ such that solving a (ordinary) least squares problem (with
any right-hand side $\vb$) on this subset of rows and corresponding right-hand
side entries yields a $O(n/k)$-approximation to the least squares problem over
all rows and right-hand side entries.
Row Sampling does not directly apply to our problem because (1) it does not
minimize over permutations of the right-hand side, and (2) the approximation
factor is too large.
However, we are able to use it to narrow the search space in our problem.

An alternative to Row Sampling is to simply enumerate all subsets of $k$ rows of
$\vX$.
This is justified by a recent result of~\citet{derezinski2017unbiased}, which
shows that for any right-hand side $\vb \in \R^n$, using ``volume
sampling''~\citep{avron2013faster} to choose a matrix $\vS \in \cbr[0]{0,1}^{k
\times k}$ (where each row has one non-zero entry) gives a similar guarantee as
that of Row Sampling, except with the $O(n/k)$ factor replaced by $k+1$ in
expectation.

\subsection{Analysis}

The approximation guarantee of \Cref{alg:approx} is given in the following
\namecref{thm:approx}.

\begin{theorem}
  \label{thm:approx}
  \Cref{alg:approx} returns $\hat{\vw} \in \R^k$ and $\hat{\vPi} \in \cP_n$
  satisfying
  \begin{equation*}
    \norm{\vX\hat{\vw} - \hat\vPi^\T\vy}_2^2
    \ \leq \ (1+\epsilon) \min_{\vw
    \in \R^k, \vPi \in \cP_n} \norm{\vX\vw - \vPi^\T\vy}_2^2
    \,.
  \end{equation*}
\end{theorem}

\begin{proof}
  Let $\opt := \min_{\vw,\vPi} \norm[0]{\vX\vw - \vPi^\T\vy}_2^2$ be the optimal
  cost, and let $(\vw_\star,\vPi_\star)$ denote a solution achieving this cost.
  The optimality implies that $\vw_\star$ satisfies the normal equations
  $\vX^\T\vX\vw_\star = \vX^\T\vPi_\star^\T\vy$.
  Observe that there exists a vector $\vb_\star \in \cB$ satisfying
  $\vS\vb_\star = \vS\vPi_\star^\T\vy$.
  By \Cref{thm:row-sampling} and the normal equations, the vector
  $\tilde{\vw}_{\vb_\star}$ and cost value $r_{\vb_\star}$ satisfy
  \begin{equation*}
    \opt
    \ \leq \
    r_{\vb_\star}
    \ \leq \
    \norm{\vX\tilde{\vw}_{\vb_\star} - \vPi_\star^\T\vy}_2^2
    \ = \
    \norm{\vX(\tilde{\vw}_{\vb_\star} - \vw_\star)}_2^2
    + \opt
    \ \leq \ c \cdot \opt \,.
  \end{equation*}
  Moreover, since $\vX^\T\vX = \vI_k$, we have that
  $\norm[0]{\tilde{\vw}_{\vb_\star} - \vw_\star}_2 \leq \sqrt{(c-1)\opt} \leq
  \sqrt{c r_{\vb_\star}}$.
  By construction of $\cN_{\vb_\star}$, there exists $\vw \in \cN_{\vb_\star}$
  satisfying $\norm[0]{\vw - \vw_\star}_2^2 = \norm[0]{\vX(\vw - \vw_\star)}_2^2
  \leq \epsilon r_{\vb_\star} / c \leq \epsilon \opt$.
  For this $\vw$, the normal equations imply
  \begin{equation*}
    \min_{\vPi \in \cP_n} \norm[0]{\vX\vw - \vPi^\T\vy}_2^2
    \ \leq \
    \norm[0]{\vX\vw - \vPi_\star^\T\vy}_2^2
    \ = \ \norm[0]{\vX(\vw - \vw_\star)}_2^2 + \opt
    \ \leq \ (1+\epsilon) \opt
    \,.
  \end{equation*}
  Therefore, the solution returned by \Cref{alg:approx} has cost no more than
  $(1+\epsilon) \opt$.
\end{proof}

By the results of \citet{pananjady2016linear} for maximum likelihood estimation,
our algorithm enjoys recovery guarantees for $\bar{\vw}$ and $\bar{\pi}$ when
the data come from the Gaussian measurement model~\eqref{eq:measurements}.
However, the approximation guarantee also holds for worst-case inputs without
generative assumptions.

\paragraph{Running time.}

We now consider the running time of \Cref{alg:approx}.
There is the initial cost for singular value decomposition (as discussed at the
beginning of the section), and also for ``Row Sampling''; both of these take
$\poly(n,d)$ time.
For the rest of the algorithm, we need to consider the size of $\cB$ and the
size of the net $\cN_{\vb}$ for each $\vb \in \cB$.
First, we have $|\cB| \leq n^r = n^{O(k)}$, since $\vS$ has only $4k$ rows and
each row has at most a single non-zero entry.
Next, for each $\vb \in \cB$, we construct the $\delta$-net $\cN_{\vb}$ (for
$\delta := \sqrt{\epsilon r_{\vb}/c}$) by constructing a $\delta/\sqrt{k}$-net
for the $\ell_\infty$-ball of radius $\sqrt{cr_{\vb}}$ centered at
$\tilde{\vw}_{\vb}$ (using an appropriate axis-aligned grid).
This has size $|\cN_{\vb}| \leq (4c^2k/\epsilon)^{k/2} = (n/\epsilon)^{O(k)}$.
Finally, each $\argmin_{\vw \in \R^k}$ computation takes $O(nk^2)$ time, and
each $(\arg)\min_{\vPi \in \cP_n}$ takes $O(nk + n \log n)$
time~\citep{pananjady2016linear} (also see \Cref{sec:approx-details}).
So, the overall running time is $(n/\epsilon)^{O(k)} + \poly(n,d)$.

%% file: lattice.tex
\section{Exact recovery algorithm in noiseless Gaussian setting}
\label{sec:lattice}

To counter the intractability of the least squares problem in~\eqref{eq:mle}
confronted in \Cref{sec:approx}, it is natural to explore distributional
assumptions that may lead to faster algorithms.
In this section, we consider the noiseless measurement model where the
$(\vx_i)_{i=1}^n$ are i.i.d.~draws from $\Normal(\v0,\vI_d)$ (as
in~\citep{pananjady2016linear}).
We give an algorithm that exactly recovers $\bar{\vw}$ with high probability
when $n \geq d+1$.
The algorithm runs in $\poly(n,d)$-time when $(\vx_i)_{i=1}^n$ and $\bar{\vw}$
are appropriately quantized.

It will be notationally simpler to consider $n+1$ covariate vectors and
responses
\begin{equation}
  y_i
  \ = \
  \bar{\vw}^\T \vx_{\bar{\pi}(i)}
  \,,
  \quad i = 0, 1, \dotsc, n
  \,.
  \label{eq:noiseless-measurements}
\end{equation}
Here, $(\vx_i)_{i=0}^n$ are $n+1$ i.i.d.~draws from $\Normal(\v0,\vI_d)$, the
unknown permutation $\bar{\pi}$ is over $\cbr[0]{0,1,\dotsc,n}$, and the
requirement of at least $d+1$ measurements is expressed as $n \geq d$.

In fact, we shall consider a variant of the problem in which we are given one of
the values of the unknown permutation $\bar{\pi}$.
Without loss of generality, assume we are given that $\bar{\pi}(0) = 0$.
Solving this variant of the problem suffices because there are only $n+1$
possible values of $\bar{\pi}(0)$: we can try them all, incurring just a factor
$n+1$ in the computation time.
So henceforth, we just consider $\bar{\pi}$ as an unknown permutation on $[n]$.

\subsection{Algorithm}

Our algorithm, shown as \Cref{alg:perm}, is based on a reduction to the Subset
Sum problem.
An instance of Subset Sum is specified by an unordered collection of source
numbers $\cbr[0]{c_i}_{i \in \cI} \subset \R$, and a target sum $t \in \R$.
The goal is to find a subset $\cS \subseteq \cI$ such that $\sum_{i \in \cS} c_i
= t$.
Although Subset Sum is NP-hard in the worst case, it is tractable for certain
structured instances~\citep{lagarias1985solving,frieze1986lagarias}.
We prove that \Cref{alg:perm} constructs such an instance with high probability.
A similar algorithm based on such a reduction was recently used by
\citet{andoni2017correspondence} for a different but related problem.

\Cref{alg:perm} proceeds by (i) solving a Subset Sum instance based on the
covariate vectors and response values (using \Cref{alg:subsetsum}), and (ii)
constructing a permutation $\hat{\pi}$ on $[n]$ based on the solution to the
Subset Sum instance.
With the permutation $\hat{\pi}$ in hand, we (try to) find a solution $\vw \in
\R^d$ to the system of linear equations $y_i = \vw^\T\vx_{\hat{\pi}(i)}$ for
$i \in [n]$.
If $\hat{\pi} = \bar{\pi}$, then there is a unique such solution almost surely.

\begin{algorithm}[t]
  \renewcommand\algorithmicrequire{\textbf{input}}
  \renewcommand\algorithmicensure{\textbf{output}}
  \caption{Find permutation}
  \label{alg:perm}
  \begin{algorithmic}[1]
    \REQUIRE
    Covariate vectors $\vx_0,\vx_1,\vx_2,\dotsc,\vx_n$ in $\R^d$;
    response values $y_0,y_1,y_2,\dotsc,y_n$ in $\R$;
    confidence parameter $\delta \in \intoo{0,1}$;
    lattice parameter $\beta > 0$.

    \renewcommand\algorithmicrequire{\textbf{assume}}
    \REQUIRE
    there exists $\bar{\vw} \in \R^d$ and permutation $\bar{\pi}$ on $[n]$ such
    that $y_i = \bar{\vw}^\T\vx_{\bar{\pi}(i)}$ for each $i \in [n]$, and
    that $y_0 = \bar{\vw}^\T\vx_0$.

    \ENSURE
    Permutation $\hat\pi$ on $[n]$ or failure.

    \STATE
    Let $\vX = [ \vx_1 | \vx_2 | \dotsb | \vx_n ]^\T \in \R^{n \times d}$, and
    its pseudoinverse be $\vX^\dag = [ \tilde{\vx}_1 | \tilde{\vx}_2 | \dotsb |
    \tilde{\vx}_n ]$.

    \STATE
    Create Subset Sum instance with $n^2$ source numbers $c_{i,j} := y_i
    \tilde{\vx}_j^\T\vx_0$ for $(i,j) \in [n] \times [n]$ and target sum
    $y_0$.

    \STATE
    Run \Cref{alg:subsetsum} with Subset Sum instance and lattice parameter
    $\beta$.

    \IF{\Cref{alg:subsetsum} returns a solution $\cS \subseteq [n] \times [n]$}

      \RETURN any permutation $\hat{\pi}$ on $[n]$ such that $\hat{\pi}(i) = j$
      implies $(i,j) \in \cS$.

    \ELSE

      \RETURN failure.

    \ENDIF

  \end{algorithmic}
\end{algorithm}

\begin{algorithm}[t]
  \renewcommand\algorithmicrequire{\textbf{input}}
  \renewcommand\algorithmicensure{\textbf{output}}
  \caption{\citet{lagarias1985solving} subset sum algorithm}
  \label{alg:subsetsum}
  \begin{algorithmic}[1]
    \REQUIRE Source numbers $\cbr[0]{ c_i }_{i \in \cI} \subset \R$;
    target sum $t \in \R$; lattice parameter $\beta>0$.

    \ENSURE
    Subset $\hat{\cS} \subseteq \cI$ or failure.

    \STATE
    Construct lattice basis
    $\vB \in \R^{(\abs{\cI}+2) \times (\abs{\cI} + 1)}$ where
    \begin{align*}
      \vB
      & \ := \
      \left[
        \begin{array}{c|c}
          \multicolumn{2}{c}{\Large\vI_{\abs{\cI}+1}} \\
          \hline
          \beta t
          &
          -\beta c_i : i \in \cI
        \end{array}
      \right]
      \ \in \ \R^{(\abs{\cI}+2) \times (\abs{\cI} + 1)}
      \,.
    \end{align*}

    \STATE
    Run basis reduction \citep[e.g.,][]{lenstra1982factoring} to find non-zero
    lattice vector $\vv$ of length at most $2^{|\cI|/2} \cdot \lambda_1(\vB)$.

    \IF{$\vv = z (1,\vchi_{\hat{\cS}}^\T,0)^\T$, with $z \in \Z$ and
    $\vchi_{\hat{\cS}} \in \cbr[0]{0,1}^{\cI}$ is characteristic vector for some
    $\hat{\cS} \subseteq \cI$}

      \RETURN $\hat{\cS}$.

    \ELSE

      \RETURN failure.

    \ENDIF

  \end{algorithmic}
\end{algorithm}

\subsection{Analysis}

The following \namecref{thm:lattice} is the main recovery guarantee for
\Cref{alg:perm}.

\begin{theorem}
  \label{thm:lattice}
  Pick any $\delta \in \intoo{0,1}$.
  Suppose $(\vx_i)_{i=0}^n$ are i.i.d.~draws from $\Normal(\v0,\vI_d)$, and
  $(y_0)_{i=1}^n$ follow the noiseless measurement model
  from~\eqref{eq:noiseless-measurements} for some $\bar{\vw} \in \R^d$ and
  permutation $\bar{\pi}$ on $[n]$ (and $\bar{\pi}(0) = 0$), and that $n\geq d$.
  Furthermore, suppose \Cref{alg:perm} is run with inputs $(\vx_i)_{i=0}^n$,
  $(y_i)_{i=0}^n$, $\delta$, and $\beta$, and also that $\beta \geq
  2^{n^2}/\veps$ where $\veps$ is defined in~\Cref{eq:veps}.
  With probability at least $1-\delta$, \Cref{alg:perm} returns $\hat{\pi} =
  \bar{\pi}$.
\end{theorem}

\begin{remark}
  The value of $\veps$ from~\Cref{eq:veps} is directly proportional to
  $\norm{\bar{\vw}}_2$, and \Cref{alg:perm} requires a lower bound on $\veps$
  (in the setting of the lattice parameter $\beta$).
  Hence, it suffices to determine a lower bound on $\norm{\bar{\vw}}_2$.
  Such a bound can be obtained from the measurement values: a standard tail
  bound (\Cref{lem:chi2-conc} in \Cref{sec:prob}) shows that with high
  probability, $\sqrt{\sum_{i=1}^n y_i^2 / (2n)}$ is a lower bound on
  $\norm[0]{\bar{\vw}}_2$, and is within a constant factor of it as well.
\end{remark}

\begin{remark}
  \Cref{alg:perm} strongly exploits the assumption of noiseless measurements,
  which is expected given the $\snr$ lower bounds of~\citet{pananjady2016linear}
  for recovering $\bar{\pi}$.
  The \namecref{alg:perm}, however, is also very brittle and very likely fails
  in the presence of noise.
\end{remark}

\begin{remark}
  The recovery result does not contradict the results of
  \citet{unnikrishnan2015unlabeled}, which show that a collection of $2d$
  measurement vectors are necessary for recovering all $\bar{\vw}$, even in the
  noiseless measurement model of~\eqref{eq:noiseless-measurements}.
  Indeed, our result shows that for a \emph{fixed} $\bar{\vw} \in \R^d$, with
  high probability $d+1$ measurements in the model
  of~\eqref{eq:noiseless-measurements} suffice to permit exactly recovery of
  $\bar{\vw}$, but this same set of measurement vectors (when $d+1 < 2d$) will
  fail for some other $\bar{\vw}'$.
\end{remark}

The proof of \Cref{thm:lattice} is based on the following
\namecref{thm:subsetsum}---essentially due to \citet{lagarias1985solving} and
\citet{frieze1986lagarias}---concerning certain structured instances of Subset
Sum that can be solved using the lattice basis reduction algorithm of
\citet{lenstra1982factoring}.
Given a basis $\vB = [ \vb_1 | \vb_2 | \dotsb | \vb_k ] \in \R^{m \times k}$ for
a lattice
\begin{equation*}
  \cL(\vB)
  \ := \
  \cbr{ \sum_{i=1}^k z_i \vb_i : z_1,z_2,\dotsc,z_k \in \Z }
  \ \subset \ \R^m
  \,,
\end{equation*}
this algorithm can be used to find a non-zero vector $\vv \in \cL(\vB) \setminus
\cbr[0]{\v0}$ whose length is at most $2^{(k-1)/2}$ times that of the shortest
non-zero vector in the lattice
\begin{equation*}
  \lambda_1(\vB)
  \ := \
  \min_{\vv \in \cL(\vB) \setminus \cbr[0]{\v0}}
  \norm{\vv}_2
  \,.
\end{equation*}

\begin{theorem}[\citep{lagarias1985solving,frieze1986lagarias}]
  \label{thm:subsetsum}
  Suppose the Subset Sum instance specified by source numbers $\cbr[0]{c_i}_{i
  \in \cI} \subset \R$ and target sum $t \in \R$ satisfy the following
  properties.
  \begin{enumerate}[leftmargin=2em]
    \item
      There is a subset $\cS^\star \subseteq \cI$ such that $\sum_{i \in
      \cS^\star} c_i = t$.

    \item
      Define
      $R := 2^{\abs{\cI}/2} \sqrt{\abs{\cS^\star}+1}$ and
      $\cZ_R
        := \cbr[0]{
          (z_0,\vz) \in \Z \times \Z^{\cI} :
          0 < z_0^2 + \sum_{i \in \cI} z_i^2 \leq R^2
        }$.
      There exists $\veps > 0$ such that
      $\abs[0]{
          z_0 \cdot t - \sum_{i \in \cI} z_i \cdot c_i
        }
        \geq
        \veps$
      for each $(z_0,\vz) \in \cZ_R$ that is not an integer multiple of
      $(1,\vchi^\star)$, where $\vchi^\star \in \cbr[0]{0,1}^{\cI}$ is the
      characteristic vector for $\cS^\star$.

  \end{enumerate}
  Let $\vB$
  be the lattice basis $\vB$ constructed by \Cref{alg:subsetsum}, and assume
  $\beta \geq 2^{|\cI|/2}/\veps$.
  Then every non-zero vector in the lattice $\Lambda(\vB)$ with length at most
  $2^{\abs{\cI}/2}$ times the length of the shortest non-zero vector in
  $\Lambda(\vB)$ is an integer multiple of the vector $(1,\vchi_{\cS^\star},0)$,
  and the basis reduction algorithm of~\citet{lenstra1982factoring} returns such
  a non-zero vector.
\end{theorem}

The Subset Sum instance constructed in \Cref{alg:perm} has $n^2$ source numbers
$\cbr[0]{ c_{i,j} : (i,j) \in [n] \times [n]}$ and target sum $y_0$.
We need to show that it satisfies the two conditions of \Cref{thm:subsetsum}.

Let $\cS_{\bar{\pi}} := \cbr[0]{ (i,j) : \bar{\pi}(i) = j } \subset [n] \times
[n]$, and let $\bar{\vPi} = (\bar{\Pi}_{i,j})_{(i,j)\in[n] \times [n]} \in
\cP_n$ be the permutation matrix with $\bar{\Pi}_{i,j} := \ind{\bar{\pi}(i) =
j}$ for all $(i,j) \in [n] \times [n]$.
Note that $\bar{\vPi}$ is the ``characteristic vector'' for $\cS_{\bar{\pi}}$.
Define $R := 2^{n^2/2} \sqrt{n+1}$ and
\begin{align*}
  \cZ_R
  & \ := \ \cbr[4]{
    (z_0,\vZ) \in \Z \times \Z^{n \times n} :
    0 < z_0^2 + \sum_{1 \leq i,j \leq n} Z_{i,j}^2 \leq R^2
  }
  \,.
\end{align*}
A crude bound shows that $\abs{\cZ_R} \leq 2^{O(n^4)}$.

The following \namecref{lem:correct-subset} establishes the first required
property in \Cref{thm:subsetsum}.

\begin{lemma}
  \label{lem:correct-subset}
  The random matrix $\vX$ has rank $d$ almost surely, and the subset
  $\cS_{\bar{\pi}}$
  satisfies $y_0 = \sum_{(i,j) \in \cS_{\bar{\pi}}} c_{i,j}$.
\end{lemma}
\begin{proof}
  That $\vX$ has rank $d$ almost surely follows from the fact that the
  probability density of $\vX$ is supported on all of $\R^{n \times d}$.
  This implies that $\vX^\dag \vX = \sum_{j=1}^n \tilde\vx_j \vx_j^\T = \vI_d$,
  and
  \begin{align*}
    y_0
    & \ = \
    \sum_{j=1}^n \vx_0^\T\tilde{\vx}_j \vx_j^\T\bar{\vw}
    \ = \
    \sum_{1 \leq i,j \leq n} \vx_0^\T\tilde{\vx}_j \cdot y_i \cdot
    \ind{\bar{\pi}(i)=j}
    \ = \
    \sum_{1 \leq i,j \leq n} c_{i,j} \cdot \ind{\bar{\pi}(i)=j}
    \,.
    \qedhere
  \end{align*}
\end{proof}

The next \namecref{lem:incorrect-coefficients} establishes the second required
property in \Cref{thm:subsetsum}.
Here, we use the fact that the Frobenius norm $\norm{z_0\bar{\vPi} - \vZ}_F$ is
at least one whenever $(z_0,\vZ) \in \Z \times \Z^{n \times n}$ is not an
integer multiple of $(1,\bar{\vPi})$.

\begin{lemma}
  \label{lem:incorrect-coefficients}
  Pick any $\eta, \eta' > 0$ such that $3\abs{\cZ_R}\eta + \eta' < 1$.
  With probability at least $1-3\abs{\cZ_R}\eta - \eta'$, every $(z_0,\vZ) \in
  \cZ_R$ with $\vZ = (Z_{i,j})_{(i,j) \in [n] \times [n]}$ satisfies
  \begin{align*}
    \abs[4]{z_0 \cdot y_0 - \sum_{i,j} Z_{i,j} \cdot c_{i,j}}
    & \ \geq \
    \frac{
      \displaystyle
      (\pi/4) \cdot \sqrt{(d-1)/n} \cdot \eta^{2+\frac1{d-1}}
    }{
      \del{\sqrt{n} + \sqrt{d} + \sqrt{2\ln(1/\eta')}}^2
    }
    \cdot \norm{z_0\bar{\vPi} - \vZ}_F \cdot \norm{\bar{\vw}}_2
    \,.
  \end{align*}
\end{lemma}
\begin{proof}
  By \Cref{lem:correct-subset}, the matrix $\bar{\vPi}$ satisfies
  $y_0 = \sum_{i,j} \bar{\Pi}_{i,j} \cdot c_{i,j}$.
  Fix any $(z_0,\vZ) \in \cZ_R$ with $\vZ = (Z_{i,j})_{(i,j) \in [n] \times
  [n]}$.
  Then
  \begin{align*}
    z_0 \cdot y_0 - \sum_{i,j} Z_{i,j} \cdot c_{i,j}
    & \ = \
    \sum_{i,j} (z_0 \cdot \bar{\Pi}_{i,j} - Z_{i,j})
    \cdot \vx_0^\T\tilde{\vx}_j
    \cdot \bar{\vw}^\T\vx_{\bar{\pi}(i)}
    \,.
  \end{align*}
  Using matrix and vector notations, this can be written compactly as the inner
  product $\vx_0^\T(\vX^\dag (z_0\bar{\vPi} - \vZ)^\T \bar{\vPi} \vX \bar{\vw} )$.
  Since $\vx_0 \sim \Normal(\v0,\vI_d)$ and is independent of $\vX$, the
  distribution of the inner product is normal with mean zero and standard
  deviation equal to $\norm[0]{\vX^\dag (z_0\bar{\vPi} - \vZ)^\T \bar{\vPi} \vX
  \bar{\vw}}_2$.
  By \Cref{lem:gaussian-anticonc} (in \Cref{sec:prob}), with probability at
  least $1-\eta$,
  \begin{align}
    \abs[1]{\vx_0^\T \del[1]{ \vX^\dag (z_0\bar{\vPi} - \vZ)^\T \bar{\vPi} \vX \bar{\vw} }}
    & \ \geq \
    \norm[0]{ \vX^\dag (z_0\bar{\vPi} - \vZ)^\T \bar{\vPi} \vX \bar{\vw} }_2
    \cdot \sqrt{\frac{\pi}{2}} \cdot \eta
    \,.
    \label{eq:ip}
  \end{align}
  Observe that $\vX^\dag = (\vX^\T\vX)^{-1} \vX^\T$ since $\vX$ has rank $d$ by
  \Cref{lem:correct-subset}, so
  \begin{align}
    \norm[0]{\vX^\dag (z_0\bar{\vPi} - \vZ)^\T \bar{\vPi} \vX \bar{\vw}}_2
    & \ \geq \
    \frac
    {\norm[0]{\vX^\T (z_0\bar{\vPi} - \vZ)^\T \bar{\vPi} \vX \bar{\vw}}_2}
    {\norm[0]{\vX}_2^2}
    \,.
    \label{eq:ratio}
  \end{align}
  By \Cref{thm:gaussian-largest-singular-values} (in \Cref{sec:prob}), with
  probability at least $1-\eta'$,
  \begin{align}
    \norm[0]{\vX}_2^2
    & \ \leq \ \del{\sqrt{n} + \sqrt{d} + \sqrt{2\ln(1/\eta')}}^2
    \,.
    \label{eq:denominator}
  \end{align}
  And by \Cref{lem:gaussian-quadratic} (in \Cref{sec:prob}), with probability at
  least $1-2\eta$,
  \begin{align}
    \norm[0]{\vX^\T (z_0\bar{\vPi} - \vZ)^\T \bar{\vPi} \vX \bar{\vw}}_2
    & \ \geq \
    \norm{(z_0\bar{\vPi} - \vZ)^\T \bar{\vPi}}_F
    \cdot \norm{\bar{\vw}}_2
    \cdot \sqrt{\frac{(d-1)\pi}{8n}}
    \cdot \eta^{1+1/(d-1)}
    \,.
    \label{eq:numerator}
  \end{align}
  Since $\bar{\vPi}$ is orthogonal, we have that $\norm[0]{(z_0\bar{\vPi} -
  \vZ)^\T \bar{\vPi}}_F
  = \norm[0]{z_0\bar{\vPi} - \vZ}_F$.
  Combining this with \eqref{eq:ip}, \eqref{eq:ratio}, \eqref{eq:denominator}, and
  \eqref{eq:numerator}, and union bounds over all $(z_0,\vZ) \in \cZ_R$ proves
  the claim.
\end{proof}

\begin{proof}[Proof of \Cref{thm:lattice}]
  \Cref{lem:correct-subset} and \Cref{lem:incorrect-coefficients} (with $\eta' :=
  \delta/2$ and $\eta := \delta/(6\abs{\cZ_R})$) together imply that with
  probability at least $1-\delta$,
  the source numbers $\cbr[0]{ c_{i,j} : (i,j) \in [n] \times [n] }$ and
  target sum $y_0$ satisfy the conditions of \Cref{thm:subsetsum} with
  \begin{align}
    \cS^\star
    & \ := \
    \cbr[0]{ (i,j) \in [n] \times [n] : \bar{\pi}(i) = j }
    \,,
    \nonumber
    \\
    \veps
    & \ := \
    \frac{
      \displaystyle
      (\pi/4) \cdot \sqrt{(d-1)/n}
      \cdot (\delta/(6\abs{\cZ_R}))^{2+\frac1{d-1}}
    }{
      \del{\sqrt{n} + \sqrt{d} + \sqrt{2\ln(2/\delta)}}^2
    }
    \cdot \norm{\bar{\vw}}_2
    \ \geq \
    2^{\displaystyle-\poly(n,\log(1/\delta))}
    \cdot \norm{\bar{\vw}}_2
    \,.
    \label{eq:veps}
  \end{align}
  Thus, in this event, \Cref{alg:subsetsum} (with $\beta$ satisfying $\beta \geq
  2^{n^2/2}/\veps$) returns $\hat{\cS} = \cS^\star$, which uniquely determines
  the permutation $\hat{\pi} = \bar{\pi}$ returned by \Cref{alg:perm}.
\end{proof}

\paragraph{Running time.}

The basis reduction algorithm of \citet{lenstra1982factoring} is iterative, with
each iteration primarily consisting of Gram-Schmidt orthogonalization and
another efficient linear algebraic process called ``size reduction''.
The total number of iterations required is
\begin{equation*}
  O\del{
    \frac{k(k+1)}{2}
    \log\del{
      \sqrt{k}
      \cdot
      \frac
      {\max_{i \in [k]} \norm{\vb_i}_2}
      {\lambda_1(\vB)}
    }
  }
  \,.
\end{equation*}
In our case, $k = n^2$ and $\lambda_1(\vB) = \sqrt{n+1}$; and by
\Cref{lem:lattice-size} (in \Cref{sec:prob}), each of the basis vectors
constructed has squared length at most $1 + \beta^2 \cdot
\poly(d,\log(n),1/\delta) \cdot \norm{\bar{\vw}}_2^2$.
Using the tight setting of $\beta$ required in \Cref{thm:lattice}, this gives a
$\poly(n,d,\log(1/\delta))$ bound on the total number of iterations as well as
on the total running time.

However, the basis reduction algorithm requires both arithmetic and rounding
operations, which are typically only available for finite precision rational
inputs.
Therefore, a formal running time analysis would require the idealized
real-valued covariate vectors $(\vx_i)_{i=0}^n$ and unknown target vector
$\bar{\vw}$ to be quantized to finite precision values.
This is doable, and is similar to using a discretized Gaussian distribution for
the distribution of the covariate vectors (and assuming $\bar{\vw}$ is a vector
of finite precision values), but leads to a messier analysis incomparable to the
setup of previous works.
Nevertheless, it would be desirable to find a different algorithm that avoids
lattice basis reduction that still works with just $d+1$ measurements.

%% file: lower.tex
\section{Lower bounds on signal-to-noise for approximate recovery}
\label{sec:lower}

In this section, we consider the measurement model from~\eqref{eq:measurements}
where $(\vx_i)_{i=1}^n$ are i.i.d.~draws from either $\Normal(\v0,\vI_d)$ or the
uniform distribution on $[-1/2,1/2]^d$, and $(\veps_i)_{i=1}^n$ are i.i.d.~draws from
$\Normal(0,\sigma^2)$.
We establish lower bounds on the signal-to-noise ratio ($\snr$),
\begin{equation*}
  \snr \ = \ \frac{\norm{\bar{\vw}}_2^2}{\sigma^2} \,,
\end{equation*}
required by any estimator $\hat{\vw} = \hat{\vw}((\vx_i)_{i=1}^n,(y_i)_{i=1}^n)$
for $\bar{\vw}$ to approximately recover $\bar{\vw}$ in expectation.
The estimators may have \emph{a priori} knowledge of the values of
$\norm{\bar{\vw}}_2$ and $\sigma^2$.

\begin{theorem}
  \label{thm:lb}
  Assume $(\veps_i)_{i=1}^n$ are i.i.d.~draws from $\Normal(0,\sigma^2)$.
  \begin{enumerate}[leftmargin=2em]
    \item
      There is an absolute constant $C>0$ such that the following holds.
      If $n\geq3$, $d\geq22$, $(\vx_i)_{i=1}^n$ are i.i.d.~draws from
      $\Normal(\v0,\vI_d)$, $(y_i)_{i=1}^n$ follow the measurement model
      from~\eqref{eq:measurements}, and
      \begin{equation*}
        \snr \ \leq \ C \cdot \min\cbr{ \frac{d}{\log\log(n)} ,\, 1 } \,,
      \end{equation*}
      then for any estimator $\hat{\vw}$, there exists some $\bar{\vw} \in \R^d$
      such that
      \begin{equation*}
        \E\sbr{ \norm{\hat{\vw} - \bar{\vw}}_2 }
        \ \geq \
        \frac1{24}
        \norm{\bar{\vw}}_2
        \,.
      \end{equation*}

    \item
      If $(\vx_i)_{i=1}^n$ are i.i.d.~draws from the uniform distribution on
      $[-1/2,1/2]^d$, and $(y_i)_{i=1}^n$ follow the measurement model
      from~\eqref{eq:measurements}, and
      \begin{equation*}
        \snr \ \leq \ 2 \,,
      \end{equation*}
      then for any estimator $\hat{\vw}$, there exists some $\bar{\vw} \in \R^d$
      such that
      \begin{equation*}
        \E\sbr{ \norm{\hat{\vw} - \bar{\vw}}_2 }
        \ \geq \
        \frac12 \del{ 1 - \frac1{\sqrt2} }
        \norm{\bar{\vw}}_2
        \,.
      \end{equation*}

  \end{enumerate}
\end{theorem}

Note that in the classical linear regression model where $y_i =
\bar{\vw}^\T\vx_i + \veps_i$ for $i \in [n]$, the maximum likelihood estimator
$\hat{\vw}_{\mle}$ satisfies $\E\norm{\hat{\vw}_{\mle} - \bar{\vw}}_2 \leq C\sigma
\sqrt{d/n}$, where $C>0$ is an absolute constant.
Therefore, the $\snr$ requirement to approximately recover $\bar{\vw}$ up
to (say) Euclidean distance $\norm{\bar{\vw}}_2/24$ is $\snr \geq 24^2C d/n$.
Compared to this setting, \Cref{thm:lb} implies that with the measurement model
of~\eqref{eq:measurements}, the $\snr$ requirement (as a function of $n$) is at
substantially higher ($d/\log\log(n)$ in the normal covariate case, or a
constant not even decreasing with $n$ in the uniform covariate case).

For the normal covariate case, \citet{pananjady2016linear} show that if $n > d$,
$\epsilon < \sqrt{n}$, and
\begin{equation*}
  \snr \ \geq \ n^{c \cdot \frac{n}{n-d} + \epsilon} \,,
\end{equation*}
then the maximum likelihood estimator $(\hat{\vw}_{\mle},\hat{\pi}_{\mle})$
(i.e., any minimizer of~\eqref{eq:mle}) satisfies $\hat{\pi}_{\mle} = \bar{\pi}$
with probability at least $1 - c' n^{-2\epsilon}$.
(Here, $c>0$ and $c'>0$ are absolute constants.)
It is straightforward to see that, on the same event, we have
$\norm{\hat{\vw}_{\mle} - \bar{\vw}}_2 \leq C\sigma\sqrt{d/n}$ for some absolute
constant $C>0$.
Therefore, the necessary and sufficient conditions on $\snr$ for approximate
recovery of $\bar{\vw}$ lie between $C' d/\log\log(n)$ and $n^{C''}$ (for
absolute constants $C', C'' > 0$).
Narrowing this range remains an interesting open problem.

A sketch of the proof in the normal covariate case is as follows.
Without loss of generality, we restrict attention to the case where $\bar{\vw}$
is a unit vector.
We construct a $1/\sqrt{2}$-packing of the unit sphere in $\R^d$; the target
$\bar{\vw}$ will be chosen from from this set.
Observe that for any distinct $\vu, \vu' \in U$, each of $(\vx_i^\T\vu)_{i=1}^n$
and $(\vx_i^\T\vu')_{i=1}^n$ is an i.i.d.~sample from $\Normal(0,1)$ of size
$n$; we prove that they therefore determine empirical distributions that are
close to each other in Wasserstein-2 distance with high probability.
We then prove that conditional on this event, the resulting distributions of
$(y_i)_{i=1}^n$ under $\bar{\vx} = \vu$ and $\bar{\vx} = \vu'$ (for any pair
$\vu, \vu' \in U$) are close in Kullback-Leibler divergence.
Hence, by (a generalization of) Fano's inequality~\citep[see,
e.g.,][]{verdu1994generalizing}, no estimator can determine the correct $\vu \in
U$ with high probability.

The proof for the uniform case is similar, using $U = \cbr{ \ve_1, -\ve_1 }$
where $\ve_1 = (1,0,\dotsc,0)^\T$.
The full proof of \Cref{thm:lb} is given in \Cref{sec:lower-proof}.

%% file: appendix-npc.tex
\section{Strong NP-hardness of the least squares problem}
\label{sec:npc}

For a vector $\vb = (b_1,b_2,\dotsc,b_n)$ and a permutation $\pi$ on $[n]$, let $\vb_\pi := (b_{\pi(1)},b_{\pi(2)},\dotsc,b_{\pi(n)})^\T$.

Recall that in the \TP problem, the input is $d = 3k$ integers $z_1,z_2,\dotsc,z_d \in \bbZ$ that sum to $Ck$ and satisfy $C/4 < z_i < C/2$ for all $i \in [d]$, and the problem is to decide if there is a partition of $[d]$ into $k$ subsets $S_1,S_2,\dotsc,S_k \subseteq [d]$ such that $|S_j| = 3$ and $\sum_{i \in S_j} z_i = C$ for each $j \in [k]$. \TP is NP-complete in the strong sense of~\citep[Section 4.2.2]{garey1979computers}.

The \PLS problem (also considered by \citet{pananjady2016linear}) is defined as follows. The input is a matrix $\vA \in \bbZ^{n \times d}$, and a vector $\vb \in \bbQ^n$. The problem is to decide if there exist a vector $\vx \in \bbQ^d$ and a permutation $\pi$ on $[n]$ such that $\vA\vx = \vb_\pi$.

\begin{proposition} \label{prop:pls-npc}
  \PLS is strongly NP-complete.
\end{proposition}
Because \PLS is equivalent to deciding if the optimal value of the least squares problem from \eqref{eq:mle} is zero, \Cref{prop:pls-npc} implies that the least squares problem from \eqref{eq:mle} is strongly NP-hard.

\begin{proof}[Proof of \Cref{prop:pls-npc}]
  It is clear that \PLS is in NP. We give an efficient reduction from \TP to \PLS. Given an instance $z_1,z_2,\dotsc,z_d$ of \TP, we construct the matrix $\vA \in \bbZ^{n \times d}$ and vector $\vb \in \bbZ^n$ with $n = d + k$ as follows:
  $$
    \vA \ := \
    \left[
      \begin{array}{cccccccccc}
        1 &   &   &   &   &   &        &   &   & \\
          & 1 &   &   &   &   &        &   &   & \\
          &   & 1 &   &   &   &        &   &   & \\
          &   &   & 1 &   &   &        &   &   & \\
          &   &   &   & 1 &   &        &   &   & \\
          &   &   &   &   & 1 &        &   &   & \\
          &   &   &   &   &   & \ddots &   &   & \\
          &   &   &   &   &   &        & 1 &   & \\
          &   &   &   &   &   &        &   & 1 & \\
          &   &   &   &   &   &        &   &   & 1 \\
        \hline
        1 & 1 & 1 &   &   &   &        &   &   & \\
          &   &   & 1 & 1 & 1 &        &   &   & \\
          &   &   &   &   &   & \ddots &   &   & \\
          &   &   &   &   &   &        & 1 & 1 & 1
      \end{array}
    \right]
    \,, \quad
    \vb \ := \
    \left[
      \begin{array}{c}
        z_1 \\
        z_2 \\
        \\
        \\
        \\
        \vdots \\
        \\
        \\
        \\
        z_d \\
        \hline
        C \\
        C \\
        \vdots \\
        C
      \end{array}
    \right]
    \,.
  $$
  The system of equations $\vA\vx = \vb_\pi$ has a solution if and only if
  $$
    b_{\pi(3j-2)} + b_{\pi(3j-1)} + b_{\pi(3j)} \ = \ C
    \,, \quad j \in [k] \,.
  $$
  Any permutation $\pi$ on $[n]$ satisfying these equations must satisfy the following two properties:
  \begin{enumerate}
    \item $\pi([d]) = [d]$.

      This holds because for $i > d$, we have $b_i = C$, and adding such $b_i$ to any other $b_{i'}$ and $b_{i''}$ gives a sum larger than $C$.

    \item $z_{\pi(3j-2)} + z_{\pi(3j-1)} + z_{\pi(3j)} = C$ for each $j \in [k]$.

      This holds because since $b_i = z_i$ for $i \in [d]$.
  \end{enumerate}
  Any permutation $\pi$ on $[n]$ with the two properties shown above gives $k$ subsets $S_j = \{ \pi(3j-2), \pi(3j-1), \pi(3j) \}$ for $j \in [k]$ such that $\sum_{i \in S_j} z_i = C$.
\end{proof}

%% file: appendix-approx.tex
\section{Additional details for approximation algorithm}
\label{sec:approx-details}

This section provides some additional details on subroutines used in
\Cref{alg:approx}.

\paragraph{Row sampling.}

First, we give the details of the ``Row Sampling'' algorithm of
\citet{boutsidis2013near} used in \Cref{sec:approx}.
The pseudocode is presented as Algorithm~\ref{alg:row-sampling}, and uses the
following notations:
\begin{itemize}[leftmargin=2em]
  \item
    For each $i \in [n]$, $\ve_i$ is the $i$-th coordinate basis vector in
    $\R^n$.

  \item
    $\displaystyle
      L(\vx, \delta_L, \vA, \ell) \ := \ \frac{\vx^{\T}(\vA - (\ell +
      \delta_L)\vI_k)^{-2} \vx}{\phi(\ell + \delta_L, \vA)- \Phi(\ell, \vA)}  -
      (\ell + \delta_L)\vI_k)^{-1} \vx \,,
    $

    where
    $\phi(\ell, \vA) := \sum_{i = 1}^k \frac{1}{\lambda_i(\vA) - \ell}$
    and $(\lambda_i(\vA))_{i=1}^k$ are the eigenvalues of $\vA$.

  \item
    $\displaystyle
      \hat U(\vx, \delta, \vB, u) \ := \ \frac{\vx^\T (\vB - u'
      \vI_{r})^{-2}\vx}{\phi'(u, \vB) - \phi'(u', \vB)} - \vx^\T (\vB - u'
      \vI_{r})^{-1}\vx \,,
    $

    where $u' = u + \delta$ and $\phi'(u, \vB) := \sum_{i = 1}^r \frac{1}{u -
    \lambda_i(\vB)}$ and $(\lambda_i(\vB))_{i=1}^k$ are the eigenvalues of
    $\vB$.

\end{itemize}

\begin{algorithm}[h]
  \renewcommand\algorithmicrequire{\textbf{input}}
  \renewcommand\algorithmicensure{\textbf{output}}
  \caption{``Row Sampling'' algorithm of \citet{boutsidis2013near}}
  \label{alg:row-sampling}
  \begin{algorithmic}[1]
    \REQUIRE
    Matrix $\vX = [ \vx_1 | \vx_2 | \dotsb | \vx_n ]^\T \in \R^{n
    \times k}$ such that $\vX^\T\vX = \vI_k$; integer $r \geq k$.

    \ENSURE
    Matrix $\vS = (S_{i,j})_{(i,j) \in [r] \times [n]} \in \R^{r \times n}$.

    \STATE
    Set $\vA_0 = \v0_{k \times k}$, $\vB_0 = \v0_{n \times n}$,
     $\vS = \v0_{r \times n}$, $\delta = (1 + n / r)(1 - \sqrt{k / r})^{-1}$ and $\delta_L = 1$.

    \FOR{$\tau = 0$ \TO $r - 1$}
    	\STATE 
	Let $\ell_\tau = \tau - \sqrt{rk}$ and $u_\tau = \delta (\tau + \sqrt{nr})$.
	\STATE
	Select $i_\tau \in [n]$ and number $t_\tau > 0$ such that 
	$\hat U(\ve_{i_\tau}, \delta, \vB_\tau, u_\tau) \leq \frac{1}{t_\tau} \leq L(\vx_{i_\tau}, \delta_L, \vA_\tau, \ell_\tau)$.
	\STATE
	Set $\vA_{\tau + 1} = \vA_{\tau}  + t_\tau \vx_{i_\tau}\vx_{i_\tau}^\T$, $\vB_{\tau + 1} = \vB_{\tau}  + t_\tau \ve_{i_\tau}\ve_{i_\tau}^\T$ and 
	$S_{\tau + 1, i_\tau} = \sqrt{r^{-1}(1 - \sqrt{k / r})} / \sqrt{t_\tau}$.
    \ENDFOR
    
    \RETURN $\vS$.
  \end{algorithmic}
\end{algorithm}

One may also consider using levarage score sampling (i.e., sample a row of $\vX$
proportional to its squared length) instead of this Row Sampling algorithm.
This would work, but would require selecting $O(k \log k)$ rows as opposed to
just $O(k)$~\citep{woodruff2014sketching}; this leads to an overall running time of $(n/\epsilon)^{O(k \log k)} + \poly(n,d)$.
Finally, as already mentioned in \Cref{sec:approx}, it also suffices to simply enumerate all $\binom{n}{k}$ subsets of $k$ rows of $\vX$. This is slower than \Cref{alg:row-sampling} but yields a better approximation guarantee (specifically, the factor $c$ from \Cref{thm:row-sampling} can be replaced by $k+1$ on account of a result of \citet{derezinski2017unbiased}). However, the overall approximation guarantee and asymptotic running time of \Cref{alg:approx} is the same.

\paragraph{One-dimensional permutation problem.}

Next, we explain how to solve the optimization problem
\begin{equation*}
  \min_{\vPi \in \cP_n} \norm{\va - \vPi^\T\vb}_2^2
\end{equation*}
for any given $\va, \vb \in \R^n$.
Let $(a_{(i)})_{i=1}^n$ denote the non-decreasing ordering $a_{(1)} \leq
a_{(2)} \leq \dotsb \leq a_{(n)}$ of the entries of $\va$, and let
$(b_{(i)})_{i=1}^n$ be analogously defined.
By \Cref{lem:w2-rep}, we have
\begin{equation*}
  \min_{\vPi \in \cP_n} \norm{\va - \vPi^\T\vb}_2^2
  \ = \
  \sum_{i=1}^n \del[1]{ a_{(i)} - b_{(i)} }^2
  \,.
\end{equation*}
Hence, if $\vPi_{\va}$ (respectively, $\vPi_{\vb}$) is the permutation matrix that rearranges the entires of
$\va$ (respectively, $\vb$) in non-decreasing order, then
\begin{equation*}
  \sum_{i=1}^n \del[1]{ a_{(i)} - b_{(i)} }^2
  \ = \
  \norm[1]{\vPi_{\va}\va - \vPi_{\vb}\vb}_2^2
  \ = \
  \norm[1]{\vPi_{\va}^\T\del{ \vPi_{\va}\va - \vPi_{\vb}\vb} }_2^2
  \ = \
  \norm[1]{\va - \vPi_{\va}^\T\vPi_{\vb}\vb }_2^2
  \,,
\end{equation*}
where the second and third equalities use the fact that permutation matrices are
orthogonal.
Thus, the minimizing permutation matrix is $\vPi = \vPi_{\vb}^\T\vPi_{\va}$.
This can be found by sorting the entries of $\va$ and of $\vb$ in $O(n \log n)$
time.

%% file: appendix-prob.tex
\section{Probability inequalities}
\label{sec:prob}

This section collects several probability inequalities used in the analysis of
\Cref{alg:perm}.
Let $\sigma_i(\vM)$ denote the $i$-th largest singular value of the matrix
$\vM$.

\paragraph{Extreme singular values of Gaussian random matrices.}

\begin{lemma}[Eq.~3.2 in \citep{rudelson2010non}]
  \label{thm:gaussian-smallest-singular-values}
  Let $\vA$ be an $n \times d$ matrix whose entries are i.i.d.~$\Normal(0,1)$
  random variables and $n\geq d$.
  For any $\eta \in \intoo{0,1}$,
  \begin{align*}
    \Pr\del{ \sigma_d(\vA) \leq \frac{\eta}{\sqrt{d}} }
    & \ \leq \ \eta
    \,.
  \end{align*}
\end{lemma}

\begin{lemma}[Theorem II.13 in \citep{davidson2001local}]
  \label{thm:gaussian-largest-singular-values}
  Let $\vA$ be an $n \times d$ matrix whose entries are i.i.d.~$\Normal(0,1)$
  random variables.
  For any $\eta \in \intoo{0,1}$,
  \begin{align*}
    \Pr\del{ \sigma_1(\vA) \geq \sqrt{n} + \sqrt{d} + \sqrt{2\ln(1/\eta)} }
    & \ \leq \ \eta
    \,.
  \end{align*}
\end{lemma}

\paragraph{Tail bounds for Gaussian and $\chi^2$ random variables.}

\begin{lemma}
  \label{lem:gaussian-conc}
  Let $Z \sim \Normal(0,1)$.
  For any $\eta \in \intoo{0,1}$, $\Pr(Z^2 \geq 2\ln(2/\delta)) \leq \eta$.
\end{lemma}
\begin{proof}
  This follows from the standard Chernoff bounding method.
\end{proof}

\begin{lemma}[Lemma 1 in \citep{laurent2000adaptive}]
  \label{lem:chi2-conc}
  Let $W \sim \chi_k^2$.
  For any $\eta \in \intoo{0,1}$, $\Pr(W \geq k + 2\sqrt{k\ln(1/\eta)} +
  2\ln(1/\eta)) \leq \eta$.
\end{lemma}

\paragraph{Anti-concentration bounds for Gaussian and $\chi^2$ random variables.}

\begin{lemma}
  \label{lem:gaussian-anticonc}
  Let $Z \sim \Normal(0,1)$.
  For any $\eta \in \intoo{0,1}$, $\Pr(Z^2 \leq \pi\eta^2/2) \leq \eta$.
\end{lemma}
\begin{proof}
  This follows from direct integration.
\end{proof}

\begin{lemma}[Lemma 9 in \citep{pananjady2016linear}]
  \label{lem:chi2-anticonc}
  Let $W \sim \chi_k^2$.
  For any $\eta \in \intoo{0,1}$, $\Pr(W \leq k\eta^{2/k}/4) \leq \eta$.
\end{lemma}

\begin{lemma}
  \label{lem:gaussian-quadratic}
  Let $\vx \in \R^d$ be any vector, $\vM \in \R^{n \times n}$ be any matrix, and
  $\vA$ a random $n \times d$ matrix of i.i.d.~$\Normal(0,1)$ random variables.
  For any $\eta \in \intoo{0,1/2}$,
  \begin{align*}
    \Pr\del{
      \norm[0]{\vA^\T \vM \vA \vx}_2
      \ \leq \
      \norm{\vM}_F
      \cdot \norm{\vx}_2
      \cdot \sqrt{\frac{(d-1)\pi}{8n}}
      \cdot \eta^{1+1/(d-1)}
    }
    & \ \leq \ 2\eta
    \,.
  \end{align*}
\end{lemma}
\begin{proof}
  Let $\vu_1 := \vx / \norm{\vx}_2$, and extend to an orthonormal basis $\vu_1,
  \vu_2, \dotsc, \vu_d$ for $\R^d$.
  Let $\vg_i := \vA\vu_i$ for each $i \in [d]$, so $\vg_1, \vg_2, \dotsc, \vg_d$
  are i.i.d.~$\Normal(\v0,\vI_n)$ random vectors.
  We first show that
  \begin{align}
    \Pr\del{
      \norm{\vM\vg_1}_2
      \ \leq \
      \norm{\vM}_F \cdot \sqrt{\frac{\pi}{2n}} \cdot \eta
    }
    & \ \leq \ \eta
    \,.
    \label{eq:g1}
  \end{align}
  To see this, note that the distribution of $\norm{\vM\vg_1}_2^2$ is the same
  as that of $\sum_{i=1}^n \sigma_i(\vM)^2 \cdot Z_i^2$, where $Z_1, Z_2,
  \dotsc, Z_n$ are i.i.d.~$\Normal(0,1)$ random variables.
  Therefore, \Cref{lem:gaussian-anticonc} and the fact $\norm{\vM}_2^2 \geq
  \norm{\vM}_F^2 / n$ proves the claim in~\eqref{eq:g1}.

  Next, observe that
  \begin{align}
    \norm[0]{\vA^\T \vM \vA \vx}_2^2
    & \ = \
    \norm{\vx}_2^2 \cdot
    \vu_1^\T \vA^\T \vM^\T \vA
    \del[4]{ \sum_{i=1}^d \vu_i \vu_i^\T }
    \vA^\T \vM \vA \vu_1
    \nonumber \\
    & \ = \
    \norm{\vx}_2^2 \cdot
    \vg_1^\T
    \vM^\T
    \del[4]{ \sum_{i=1}^d \vg_i \vg_i^\T }
    \vM \vg_1
    \nonumber \\
    & \ \geq \
    \norm{\vx}_2^2 \cdot
    \sum_{i=2}^d \del{\vg_i^\T\vM\vg_1}^2
    \,.
    \label{eq:chi2}
  \end{align}
  Conditional on $\vg_1$, the final right-hand side in~\eqref{eq:chi2} has the
  same distribution as $\norm{\vx}_2^2 \cdot \norm{\vM\vg_1}_2^2 \cdot W$, where
  $W \sim \chi_{d-1}^2$ is a chi-squared random variable with $d-1$ degrees of
  freedom.
  Therefore, \Cref{lem:chi2-anticonc} implies
  \begin{align*}
    \Pr\del{
      \norm[0]{\vA^\T \vM \vA \vx}_2
      \ \leq \
      \norm{\vx}_2 \cdot \norm{\vM\vg_1}_2 \cdot \frac{\sqrt{d-1}}{2} \cdot
      \eta^{1/(d-1)}
    }
    & \ \leq \
    \eta
    \,.
  \end{align*}
  Combining this inequality with the inequality from~\eqref{eq:g1} and a union
  bound proves the claim.
\end{proof}

\paragraph{Lattice basis size.}

The following \namecref{lem:lattice-size} is used to bound the size of the
lattice basis vectors constructed by \Cref{alg:perm} (via \Cref{alg:subsetsum}).
Recall that there are $n^2+1$ basis vectors; one has length $\sqrt{1 +
\beta^2y_0^2}$, and the remaining $n^2$ have length $\sqrt{1 +
\beta^2c_{i,j}^2}$ for $(i,j) \in [n] \times [n]$.

\begin{lemma}
  \label{lem:lattice-size}
  For any $\eta \in \intoo{0,1/5}$, with probability at least $1-5\eta$,
  \begin{align*}
    |y_0|
    & \ \leq \
    \norm[0]{\bar{\vw}}_2 \sqrt{2\ln(2/\eta)}
    \,,
    \\
    \abs{c_{i,j}}
    & \ \leq \
    \norm[0]{\bar{\vw}}_2
    \cdot \sqrt{2\ln(2n/\eta)}
    \cdot \frac{d}{\eta^2}
    \cdot \sqrt{d + 2\sqrt{d\ln(n/\eta)} + 2\ln(n/\eta)}
    \cdot \sqrt{2\ln(2n/\eta)}
    \,,
    \quad
    (i,j) \in [n] \times [n]
    \,.
  \end{align*}
\end{lemma}
\begin{proof}
  By \Cref{thm:gaussian-smallest-singular-values}, \Cref{lem:gaussian-conc}, and
  \Cref{lem:chi2-conc}, with probability at least $1-5\eta$,
  \begin{align*}
    \norm[1]{(\vX^\T\vX)^{-1}}_2
    & \ \leq \ \frac{d}{\eta^2}
    \,,
    \\
    \abs[0]{\vx_0^\T\bar{\vw}}
    & \ \leq \
    \norm[0]{\bar{\vw}}_2 \sqrt{2\ln(2/\eta)}
    \,,
    \\
    \abs[0]{\vx_{\bar{\pi}(i)}^\T\bar{\vw}}
    & \ \leq \
    \norm[0]{\bar{\vw}}_2 \sqrt{2\ln(2n/\eta)}
    \,,
    \quad
    i \in [n]
    \,,
    \\
    \abs[0]{\tilde{\vx}_j^\T\vx_0}
    & \ \leq \
    \norm[0]{\tilde{\vx}_j}_2 \sqrt{2\ln(2n/\eta)}
    \,,
    \quad j \in [n]
    \,,
    \\
    \norm[0]{\vx_j}_2
    & \ \leq \
    \sqrt{d + 2\sqrt{d\ln(n/\eta)} + 2\ln(n/\eta)}
    \,,
    \quad
    j \in [n]
    \,.
  \end{align*}
  In this event, we have for each $(i,j) \in [n] \times [n]$,
  \begin{align*}
    \abs{c_{i,j}}
    & \ = \
    \abs[0]{\vx_{\bar{\pi}(i)}^\T\bar{\vw}}
    \cdot \abs[0]{\tilde{\vx}_j^\T\vx_0}
    \\
    & \ \leq \
    \norm[0]{\bar{\vw}}_2
    \cdot \sqrt{2\ln(2n/\eta)}
    \cdot \norm[0]{\vX^\dag\ve_j}_2
    \cdot \sqrt{2\ln(2n/\eta)}
    \\
    & \ = \
    \norm[0]{\bar{\vw}}_2
    \cdot \sqrt{2\ln(2n/\eta)}
    \cdot \norm[0]{(\vX^\T\vX)^{-1}\vX^\T\ve_j}_2
    \cdot \sqrt{2\ln(2n/\eta)}
    \\
    & \ \leq \
    \norm[0]{\bar{\vw}}_2
    \cdot \sqrt{2\ln(2n/\eta)}
    \cdot \frac{d}{\eta^2}
    \cdot \sqrt{d + 2\sqrt{d\ln(n/\eta)} + 2\ln(n/\eta)}
    \cdot \sqrt{2\ln(2n/\eta)}
    \,,
  \end{align*}
  and $\abs[0]{y_0} \leq \norm[0]{\bar{\vw}}_2 \sqrt{2\ln(2/\eta)}$.
\end{proof}

%% file: appendix-lower.tex
\section{Proof of signal-to-noise lower bounds}
\label{sec:lower-proof}

This section provides the proof of \Cref{thm:lb}.

Below, for any vector $\va = (a_1,a_2,\dotsc,a_n)^\T$, we use the notation
$(a_{(i)})_{i=1}^n$ to denote the non-decreasing ordering $a_{(1)} \leq a_{(2)}
\leq \dotsb \leq a_{(n)}$ of its entries, and $(\va)^\uparrow :=
(a_{(1)},a_{(2)},\dotsc,a_{(n)})^\T$ to denote the vector of the entries in this
order.

We use the following representation for the Kantorovich transport distance with
respect to Euclidean metric (i.e., Wasserstein-2 distance, denoted by $W_2$).

\begin{lemma}[Lemma 4.1 in~\citep{bobkov2014one}]
  \label{lem:w2-rep}
  Let $\mu_n$ be the empirical measure on $a_1, a_2, \dotsc, a_n \in \R$, and
  $\nu_n$ be the empirical measure on $b_1, b_2, \dotsc, b_n \in \R$.
  Then
  \begin{equation*}
    W_2(\mu_n,\nu_n)^2
    \ = \
    \min_\pi
    \frac1n \sum_{i=1}^n
    (a_i - b_{\pi(i)})^2
    \ = \
    \frac1n \sum_{i=1}^n
    (a_{(i)} - b_{(i)})^2
    \,,
  \end{equation*}
  where $\min_\pi$ denotes minimization over permutations $\pi$ on $[n]$.
\end{lemma}

For probability measures $\mu$ and $\nu$, we use $\KL(\mu,\nu)$ to denote the
Kullback-Leibler divergence between $\mu$ and $\nu$, and $\norm{\mu -
\nu}_{\tv}$ to denote the total variation distance between $\mu$ and $\nu$.

Since $\bar{\pi}$ is unknown in the measurement model
from~\eqref{eq:measurements}, we may assume that $y_1, y_2, \dotsc, y_n$ are
provided as an unordered multiset, denoted by $\bag{y_i}_{i=1}^n$.
In fact, we shall use the following equivalent generative process:
\begin{enumerate}
  \item
    Draw $(\vx_i)_{i=1}^n$ i.i.d.~from either $\Normal(\v0,\vI_d)$ (in \Cref{sec:lower-proof-normal}) or the uniform distribution on $[-1/2,1/2]^d$ (in \Cref{sec:lower-proof-uniform}), and independently, draw $\vveps \sim \Normal(\v0,\sigma^2\vI_n)$.

  \item
    Set $\vh_{\bar{\vw}} := (\bar{\vw}^\T \vx_1,\bar{\vw}^\T \vx_2,\dotsc,\bar{\vw}^\T \vx_n)^\T$.

  \item
    Set $\vy := \vh_{\bar{\vw}}^\uparrow + \vveps$.

\end{enumerate}
It is clear that $((\vx_i)_{i=1}^n,\bag{y_i}_{i=1}^n)$ has the same distribution under this model as under that from~\eqref{eq:measurements}.

\subsection{Normal case}
\label{sec:lower-proof-normal}

We first consider the case where $(\vx_i)_{i=1}^n$ are i.i.d.~draws from
$\Normal(\v0,\vI_d)$.
By homogeneity, we may assume without loss of generality that
$\norm{\bar{\vw}}_2 = 1$, so $\snr = 1/\sigma^2$.

The proof is based on the Generalized Fano method of
\citet{verdu1994generalizing} as described by \citet{yu1997assouad}.
\begin{lemma}[Lemma 3 in~\citep{yu1997assouad}]
  \label{lem:fano}
  Let $(\Theta,\rho)$ be a pseudometric space, and let $\widetilde{\Theta}
  \subseteq \Theta$ index a collection of probability measures
  $(P_\theta)_{\theta \in \widetilde{\Theta}}$ such that $\rho(\theta, \theta')
  \geq \alpha$ and $\KL(P_{\theta}, P_{\theta}) \leq \beta$ for all distinct
  pairs $\theta, \theta' \in \widetilde{\Theta}$.
  Then for any estimator $\hat{\theta}$ taking values in $\Theta$,
  \begin{equation*}
    \max_{\theta \in \widetilde{\Theta}} \E_{P_{\theta}}
    \sbr[1]{
      \rho(\hat{\theta},\theta)
    }
    \ \geq \
    \frac{\alpha}{2}
    \del{ 1 - \frac{\beta + \ln2}{\ln |\widetilde{\Theta}|} }
    \,,
  \end{equation*}
  where $\E_{P_{\theta}}$ denotes expectation with respect to data drawn from
  $P_{\theta}$.
\end{lemma}

We apply \Cref{lem:fano} with $(\Theta,\rho) = (S^{d-1},\norm{\cdot}_2)$.
We construct a packing $U$ of the unit sphere $S^{d-1} := \cbr[0]{ \vu \in \R^d
: \norm{\vu}_2 = 1 }$ using the following variant of the Gilbert-Varshamov
bound.

\begin{lemma}[Lemma 4.10 in~\citep{massart2007concentration}]
  \label{lem:gv}
  For every $h \in [d]$ such that $h \leq d/4$, there exists a subset $C$ of
  $\cbr[0]{0,1}^d$ such that (i) the Hamming weight of each $\vc \in C$ is $h$,
  (ii) the Hamming distance between every distinct pair $\vc, \vc' \in C$ is
  more than $h/2$, and (iii) the cardinality of $C$ satisfies $\ln|C| \geq 0.233
  h \ln(d/h)$.
\end{lemma}

We take $C \subseteq \cbr{0,1}^d$ as guaranteed by \Cref{lem:gv}
with $h := \lfloor d/4 \rfloor$, and let
\begin{equation*}
  U \ := \ \cbr{ \vc / \sqrt{h} : \vc \in C } \ \subset \ S^{d-1} \,.
\end{equation*}
Observe that $U$ is a $(1/\sqrt{2})$-packing of $S^{d-1}$ (i.e., every distinct
pair $\vu, \vu' \in U$ satisfies $\norm{\vu-\vu'}_2 > 1/\sqrt{2}$), and
\begin{equation*}
  \ln |U|
  \ \geq \
  0.233 \del{ \frac{d}{4} - 1 } \ln 4
  \,.
\end{equation*}

For each $\vu \in U$, let $P_{\vu}$ denote the probability distribution of
$((\vx_i)_{i=1}^n, \bag{y_i}_{i=1}^n)$ when $\bar{\vw} = \vu$.
Also, define $Q_{\vu}$ to be the corresponding conditional distribution of
$\bag{y_i}_{i=1}^n$ given $(\vx_i)_{i=1}^n$, and $\tilde Q_{\vu}$ to be the
corresponding conditional distribution of $\vy$ given $(\vx_i)_{i=1}^n$.

For any $\vu, \vu' \in U$,
\begin{equation}
  \KL(Q_{\vu},Q_{\vu'})
  \ \leq \
  \KL(\tilde Q_{\vu},\tilde Q_{\vu'})
  \ = \
  \frac1{2\sigma^2}
  \norm{ \vh_{\vu}^\uparrow - \vh_{\vu'}^\uparrow }_2^2
  \label{eq:kl-data-processing}
\end{equation}
by the data processing inequality for $\KL$-divergence and the properties of the multivariate Gaussian distribution.
We define $\cE$ to be the event in which
\begin{equation*}
  \norm{ \vh_{\vu}^\uparrow - \vh_{\vu'}^\uparrow }_2^2
  \ \leq \
  \del{\sqrt{C_0 \log\log(n)} + \sqrt{8 \ln(|U|^2)} }^2
\end{equation*}
for all distinct $\vu, \vu' \in U$, where $C_0>0$ is the absolute constant from
\Cref{lem:w2-normal-emp} (below).
By~\Cref{eq:kl-data-processing}, \Cref{lem:w2-normal-emp}, and a union bound, we
have $\Pr(\cE) \geq 1/2$.
Therefore, by \Cref{lem:fano}, for any estimator $\hat{\vw}$,
\begin{align*}
  \max_{\vu \in U} \E_{P_{\vu}}
  \sbr{
    \norm{\hat{\vw}-\vu}_2
  }
  & \ \geq \
  \max_{\vu \in U} \E_{P_{\vu}}
  \sbr{
    \norm{\hat{\vw}-\vu}_2
    \mid \cE
  } \cdot \Pr(\cE)
  \\
  & \ \geq \
  \frac1{2\sqrt{2}}
  \del{
    1 -
    \frac{
      C_0 \log\log(n) + 16 \ln|U|
    }{
      \sigma^2\ln|U|
    }
    -
    \frac{\ln2}{\ln|U|}
  }
  \cdot
  \frac12
  \\
  & \ = \
  \frac1{4\sqrt{2}}
  \del{
    1 - \frac{C_0\log\log(n)}{\sigma^2\ln|U|} - \frac{16}{\sigma^2} -
    \frac{\ln2}{\ln|U|}
  }
  \,.
\end{align*}
Plugging in the lower bound for $\ln|U|$ and the upper bound on $\snr =
1/\sigma^2$ completes the proof.
\hfill\qed

\subsection{Uniform case}
\label{sec:lower-proof-uniform}

We now consider the case where $(\vx_i)_{i=1}^n$ are drawn i.i.d.~from the
uniform distribution on $[-1/2,1/2]^d$.\footnote{%
  We actually just need that the marginal distribution of the first coordinate of
  each $\vx_i$ be uniform on $[-1/2,1/2]$.%
}
Again, by homogeneity, we assume without loss of generality that
$\norm{\bar{\vw}}_2 = 1$, so $\snr = 1/\sigma^2$.

The proof is based on the two-point method of \citet{lecam1973convergence} as
described by \citet{yu1997assouad}.
\begin{lemma}[Lemma 1 in~\citep{yu1997assouad}]
  \label{lem:lecam}
  Let $(\Theta,\rho)$ be a pseudometric space, and let $\theta_1, \theta_2 \in
  \Theta$ correspond to probability measures $P_{\theta_1}$ and $P_{\theta_2}$
  on the same space.
  Then for any estimator $\hat{\theta}$ taking values in $\Theta$,
  \begin{equation*}
    \max_{\theta \in \cbr{\theta_1,\theta_2}} \E_{P_{\theta}}
    \sbr[1]{
      \rho(\hat{\theta},\theta)
    }
    \ \geq \
    \frac{1}{2}
    \rho(\theta_1,\theta_2)
    \del{ 1 - \norm{P_{\theta_1} - P_{\theta_2}}_{\tv} }
    \,,
  \end{equation*}
  where $\E_{P_{\theta}}$ denotes expectation with respect to data drawn from
  $P_{\theta}$.
\end{lemma}

We apply \Cref{lem:lecam} with $(\Theta,\rho) = (S^{d-1},\norm{\cdot}_2)$.
As before, we define for each $\vu \in \cbr[0]{ \ve_1, -\ve_1 }$:
\begin{itemize}
  \item
    $P_{\vu}$, the distribution of $((\vx_i)_{i=1}^n,\bag{y_i}_{i=1}^n)$ when
    $\bar{\vw} = \vu$;

  \item
    $Q_{\vu}$, the corresponding conditional distribution of
    $\bag{y_i}_{i=1}^n$ given $(\vx_i)_{i=1}^n$;

  \item
    $\tilde Q_{\vu}$, the corresponding conditional distribution of $\vy$ given
    $(\vx_i)_{i=1}^n$.

\end{itemize}

Let $\cE$ be the event in which
\begin{equation*}
  \norm{\vh_{\ve_1}^\uparrow - \vh_{-\ve_1}^\uparrow}_2^2
  \ \leq \
  1
  \,.
\end{equation*}
By \Cref{lem:w2-uniform-emp} (below), $\Pr(\cE) \geq 1/2$.
Moreover, since $P_{\ve_1}(\cE) = P_{-\ve_1}(\cE) = \Pr(\cE)$,
\begin{align*}
  \norm{P_{\ve_1} - P_{-\ve_1}}_{\tv}
  & \ \leq \
  \norm{P_{\ve_1}(\cdot \mid \cE) - P_{-\ve_1}(\cdot \mid \cE)}_{\tv}
  \Pr(\cE)
  + (1 - \Pr(\cE))
  \\
  & \ \leq \
  \sqrt{ \frac12 \KL(P_{\ve_1}(\cdot \mid \cE), P_{-\ve_1}(\cdot \mid \cE)) }
  \Pr(\cE)
  + (1 - \Pr(\cE))
  \\
  & \ \leq \
  \sqrt{\frac12 \cdot \frac1{2\sigma^2}} \Pr(\cE) + (1 - \Pr(\cE))
  \\
  & \ \leq \
  \frac12\del{ 1 + \frac1{\sqrt2} }
  \,.
\end{align*}
Above, the second inequality follows from Pinsker's inequality; the third
inequality uses \eqref{eq:kl-data-processing} and the fact
$\norm[0]{\vh_{\ve_1}^\uparrow - \vh_{-\ve_1}^\uparrow}_2^2 \leq 1$ on the
event $\cE$, the fourth inequality uses the assumption that $\snr = 1/\sigma^2
\leq 2$ and the fact $\Pr(\cE) \geq 1/2$.
We conclude by \Cref{lem:lecam} that
\begin{equation*}
  \max_{\vu \in \cbr{\ve_1,-\ve_2}} \E_{P_{\vu}}
  \sbr{
    \norm{\hat{\vw}-\vu}_2
  }
  \ \geq \
  \frac12 \cdot 2 \cdot \del{ 1 - \frac12\del{ 1 + \frac1{\sqrt2} } }
  \ = \
  \frac12 \del{ 1 - \frac1{\sqrt2} }
  \,,
\end{equation*}
completing the proof.
\hfill\qed

\subsection{Auxiliary results}

\begin{lemma}
  \label{lem:w2-normal-emp}
  There is an absolute constant $C_0>0$ such that the following holds.
  Let $n \geq 3$, and let $\vX$ be a random $n \times d$ matrix of
  i.i.d.~$\Normal(0,1)$ random variables.
  For any unit vectors $\vu, \vu' \in S^{d-1}$ and $\delta \in \intoo{0,1}$,
  \begin{equation*}
    \Pr\del{
      \norm{(\vX\vu)^\uparrow - (\vX\vu')^\uparrow}_2
      \geq \sqrt{C_0 \log\log(n)} + \sqrt{8 \ln(1/\delta)}
    }
    \ \leq \
    \delta
    \,.
  \end{equation*}
\end{lemma}

The proof of \Cref{lem:w2-normal-emp} uses the following lemmas.

\begin{lemma}[Corollary 6.14 in~\citep{bobkov2014one}]
  \label{lem:w2-normal}
  There is an absolute constant $C>0$ such that the following holds.
  If $n\geq3$, $\mu$ is the standard Gaussian measure on $\R$, and $\mu_n$ is
  the empirical measure for a size-$n$ i.i.d.~sample from $\mu$, then
  \begin{equation*}
    \E\sbr{ W_2(\mu_n,\mu)^2 }
    \ \leq \
    \frac{C \log\log(n)}{n}
    \,.
  \end{equation*}
\end{lemma}

\begin{lemma}[Eq.~2.35 in~\citep{ledoux2000concentration}]
  \label{lem:normal-conc}
  Let $\vZ \sim \Normal(\v0,\vI_p)$ be a standard normal random vector in
  $\R^p$, and $f \colon \R^p \to \R$ be $L$-Lipschitz with respect to the
  Euclidean metric.
  Then for any $t>0$,
  \begin{equation*}
    \Pr\del{ f(\vZ) \geq \E f(\vZ) + t }
    \ \leq \ e^{-t^2/(2L^2)} \,.
  \end{equation*}
\end{lemma}

\begin{proof}[Proof of \Cref{lem:w2-normal-emp}]
  Fix unit vectors $\vu$ and $\vu'$.
  Observe that the entries of each of $\vX\vu$ and $\vX\vu'$ comprises an
  i.i.d.~sample from $\Normal(0,1) =: \mu$; let $\mu_n$ and $\nu_n$ denote the
  respective empirical measures.
  Define the function $f \colon \bbR^{n \times d} \to \bbR$ by
  \begin{equation*}
    f(\vA) \ := \ 
    \norm{(\vA\vu)^\uparrow - (\vA\vu')^\uparrow}_2
    \,.
  \end{equation*}
  Then, by \Cref{lem:w2-rep}, the triangle inequality, Jensen's inequality, and
  \Cref{lem:w2-normal},
  \begin{equation*}
    \frac{\E f(\vX)}{\sqrt{n}}
    \ = \
    \E W_2(\mu_n,\nu_n)
    \ \leq \
    \E W_2(\mu_n,\mu)
    + \E W_2(\nu_n,\mu)
    \ \leq \
    2 \sqrt{ \E W_2(\mu_n,\mu)^2 }
    \ \leq \
    \sqrt{\frac{C_0 \log\log(n)}{n}}
    \,.
  \end{equation*}
  Moreover, for any $\vA, \vA' \in \R^{n \times d}$,
  \begin{align*}
    f(\vA) - f(\vA')
    & \ \leq \
    \norm{(\vA\vu)^\uparrow - (\vA\vu')^\uparrow - (\vA'\vu)^\uparrow + (\vA'\vu')^\uparrow}_2
    \\
    & \ \leq \
    \norm{(\vA\vu)^\uparrow - (\vA'\vu)^\uparrow}_2
    +
    \norm{(\vA\vu')^\uparrow - (\vA'\vu')^\uparrow}_2
    \\
    & \ \leq \
    \norm{\vA\vu - \vA'\vu}_2
    +
    \norm{\vA\vu' - \vA'\vu'}_2
    \\
    & \ \leq \
    2\norm{\vA-\vA'}_F
    \,,
  \end{align*}
  where the first two steps follow from the triangle inequality, the third step
  uses \Cref{lem:w2-rep}, and $\norm{\cdot}_F$ denotes the Frobenius norm.
  Therefore, $f$ is $2$-Lipschitz with respect to the Euclidean metric on $\R^{n
  \times d}$.
  By \Cref{lem:normal-conc}, for any $\delta \in \intoo{0,1}$,
  \begin{equation*}
    \Pr\del{ f(\vX) \geq \E f(\vX) + \sqrt{8\ln(1/\delta)} } \ \leq \ \delta
    \,.
  \end{equation*}
  Combining this with the upper bound on $\E f(\vX)$ completes the proof.
\end{proof}

\begin{lemma}[Eqs.~1.7.3 and 1.7.5 in \citep{reiss2012approximate}]
  \label{lem:uniform-order}
  Let $X_1, X_2, \dotsc, X_n$ be i.i.d.~draws from the uniform distribution on
  $[0,1]$.
  For any $r \in [n]$,
  \begin{equation*}
    \E[ X_{(r)} ] \ = \ \frac{r}{n+1} \,,
  \end{equation*}
  and for any $r, s \in [n]$ with $r \leq s$,
  \begin{equation*}
    \cov(X_{(r)},X_{(s)}) \ = \ \frac{r}{n+1} \cdot \del{ 1 - \frac{s}{n+1} }
    \cdot \frac{1}{n+2} \,.
  \end{equation*}
\end{lemma}

\begin{lemma}
  \label{lem:w2-uniform-emp}
  Let $U_1, U_2, \dotsc, U_n$ be i.i.d.~draws from the uniform distribution on
  $[-1/2,1/2]$.
  Then
  \begin{equation*}
    \Pr\del{
      \sum_{i=1}^n
      \del{ U_{(1)} + U_{(n+1-i)} }^2
      \geq 1
    }
    \ \leq \
    \frac12
    \,.
  \end{equation*}
\end{lemma}
\begin{proof}
  It suffices to show the expectation bound
  \begin{equation*}
    \E\sbr{
      \sum_{i=1}^n
      \del{ U_{(1)} + U_{(n+1-i)} }^2
    }
    \ \leq \
    \frac12
    \,,
  \end{equation*}
  since the claim then follows by Markov's inequality.
  Expanding the square and using linearity of expectation gives
  \begin{align*}
    \E\sbr{
      \sum_{i=1}^n
      \del{ U_{(1)} + U_{(n+1-i)} }^2
    }
    & \ = \
    2 \sum_{i=1}^n \E\sbr{ U_i^2 }
    + 2 \sum_{i=1}^n \E\sbr{ U_{(i)} U_{(n+1-i)} }
    \\
    & \ = \
    \frac{n}{6}
    + 2 \sum_{i=1}^n \E\sbr{ U_{(i)} U_{(n+1-i)} }
    \,.
  \end{align*}
  By \Cref{lem:uniform-order}, we have for $i \leq (n+1)/2$,
  \begin{equation*}
    \E\sbr{ U_{(i)} U_{(n+1-i)} }
    \ = \
    - \del{ \frac{i}{n+1} - \frac12 }^2 + \frac{i^2}{(n+1)^2(n+2)}
    \,,
  \end{equation*}
  and for $i > (n+1)/2$,
  \begin{equation*}
    \E\sbr{ U_{(i)} U_{(n+1-i)} }
    \ = \
    - \del{ \frac{i}{n+1} - \frac12 }^2 + \frac{(n+1-i)^2}{(n+1)^2(n+2)}
    \,.
  \end{equation*}
  Plugging-in and simplifying gives
  \begin{equation*}
    \E\sbr{
      \sum_{i=1}^n
      \del{ U_{(1)} + U_{(n+1-i)} }^2
    }
    \ = \
    \begin{cases}
      \frac12 \del{ 1 - \frac1{n+1} } & \text{if $n$ is even} \,, \\
      \frac12 \del{ 1 - \frac1{n+2} } & \text{if $n$ is odd} \,.
    \end{cases}
  \end{equation*}
\end{proof}